\documentclass[preprint,12pt]{elsarticle}

\usepackage{amsmath,amssymb,amsfonts,amsthm}
\usepackage{booktabs}
\usepackage{graphicx}
\usepackage{subcaption}
\usepackage{multirow}
\usepackage{array}
\usepackage{xcolor}
\newtheorem{assumption}{Assumption}
\newtheorem{lemma}{Lemma}
\newtheorem{theorem}{Theorem}
\newtheorem{proposition}{Proposition}
\newtheorem{corollary}{Corollary}
\newtheorem{remark}{Remark}

\journal{Neurocomputing}

\begin{document}

\begin{frontmatter}

\title{Behavior-Induced Mirror-Prox Temporal-Difference Learning for Faster Off-Policy Prediction}

\author[aff1]{Xingguo Chen\corref{cor1}}
\ead{chenxg@njupt.edu.cn}
\author[aff1]{Yuchen Shen}
\author[aff1]{Shangdong Yang}
\author[aff1]{Chao Li}
\author[aff2]{Guang Yang}
\author[aff3]{Wenhao Wang}
\cortext[cor1]{Corresponding author.}

\affiliation[aff1]{organization={Nanjing University of Posts and Telecommunications},
            city={Nanjing},
            country={China}}

\affiliation[aff2]{organization={Department of Computer Science and Technology, Nanjing University},
            city={Nanjing},
            country={China}}

\affiliation[aff3]{organization={College of Electronic Countermeasure, National University of Defense Technology},
            city={Hefei},
            country={China}}

\begin{abstract}\sloppy
Gradient temporal-difference methods provide stable off-policy prediction with linear function approximation, but their practical performance is strongly affected by the geometry induced by the auxiliary-variable metric. Existing Mirror-Prox TD methods typically use the feature covariance metric, whereas hybrid TD methods suggest that behavior-policy transition information can provide a more informative update geometry. This paper proposes a behavior-induced Mirror-Prox temporal-difference method, called STHTD-MP, which replaces the covariance metric in the primal-dual saddle-point formulation with the symmetric part of the behavior-policy Bellman matrix. The method keeps a single learning rate for the primal and auxiliary variables and applies a Mirror-Prox prediction-correction step to the resulting hybrid saddle-point operator. We provide a formal convergence analysis for fixed-policy linear prediction under standard stochastic approximation assumptions: the behavior-induced metric is positive definite, the joint mean system is Hurwitz, boundedness follows from a Lyapunov argument, and the stochastic recursion converges by the ODE method. We further derive projected-oracle ergodic gap bounds and an exact mean-operator comparison with GTD2-MP based on the spectral radius of the deterministic Mirror-Prox error matrix. The analysis shows that STHTD-MP can have a smaller mean contraction factor than GTD2-MP when the behavior-induced metric improves the saddle-point geometry. Exact numerical mean-operator analysis on two-state, Random Walk, and Boyan Chain benchmarks supports this condition, while Baird's counterexample is identified as a singular boundary case where the strict assumptions fail. Experiments over 100 independent runs report means and standard deviations of two scalar summaries -- the steady-state AUC, defined as the time-average RMSVE over the last 50\% of each trajectory, and the final RMSVE -- and show that STHTD-MP is competitive with strong online TD baselines and that its empirical advantage is geometry- and horizon-dependent.
\end{abstract}


\end{frontmatter}

\section{Introduction}

Temporal-difference (TD) learning is a central mechanism for value prediction in reinforcement learning \cite{sutton1988learning,sutton2018reinforcement}. In off-policy prediction with function approximation, however, classical TD may diverge because sampling, bootstrapping, and approximation interact in an unstable way \cite{baird1995residual,tsitsiklis1997analysis}. This has led to gradient-TD algorithms such as GTD, GTD2, TDC, and their extensions, which restore stability by introducing an auxiliary variable and optimizing projected Bellman-error objectives \cite{sutton2008convergent,sutton2009fast,maei2010gq,szepesvari2010toward}.

The stability of gradient-TD methods does not by itself settle the question of fast learning. Two issues remain particularly important. First, many gradient-TD algorithms use separate primary and auxiliary learning rates, making relative step-size tuning nontrivial. Saddle-point and proximal views of TD learning provide a single-timescale alternative by rewriting policy evaluation as a primal-dual problem \cite{liu2015finite,liu2018proximal}. Second, even within the same saddle-point framework, the auxiliary metric determines the geometry of the mean operator and can strongly affect convergence speed. Standard GTD2-type saddle-point methods use the feature covariance metric $C=\mathbb{E}_\mu[\phi\phi^\top]$. In contrast, hybrid TD ideas suggest that the behavior-policy transition matrix contains useful geometric information that is not captured by $C$ alone \cite{hackman2012faster}.

This paper asks whether behavior-policy transition information can be used to obtain a faster Mirror-Prox TD method. We answer this question by replacing the covariance metric in the primal-dual TD objective with the symmetric behavior-induced Bellman metric $H=\operatorname{sym}(A_\mu)$ and applying a Mirror-Prox prediction-correction step \cite{nemirovski2004prox,juditsky2011solving}. The resulting method, STHTD-MP, can be viewed as a behavior-induced counterpart of GTD2-MP. Both methods use Mirror-Prox, but they shape the saddle-point operator with different metrics.

The paper makes the following contributions.

\begin{itemize}
  \item We derive STHTD-MP, a single-timescale behavior-induced Mirror-Prox TD method for off-policy linear prediction. The method uses the symmetric behavior-policy Bellman matrix as the auxiliary metric.
  \item We prove that the behavior-induced metric is positive definite under standard finite-state assumptions and that the resulting joint mean system is Hurwitz. Under standard stochastic approximation conditions, the underlying single-timescale hybrid recursion converges to the projected Bellman fixed point.
  \item We provide a convergence-speed analysis beyond big-$O$ notation. In addition to stochastic ergodic gap bounds, we derive an exact mean-operator comparison with GTD2-MP through the spectral radius of the deterministic Mirror-Prox error matrix.
  \item We compute the exact finite-state mean operators for four benchmarks. The numerical spectral analysis shows that STHTD-MP has a smaller deterministic mean contraction factor than GTD2-MP on the two-state, Random Walk, and Boyan Chain problems, while Baird's counterexample is a singular boundary case where the strict assumptions fail.
  \item We conduct stochastic experiments with stronger baselines, disjoint tuning/evaluation seeds, and 100 independent evaluation runs. The results show that STHTD-MP is competitive among online first-order TD methods and that the empirical benefit of the behavior-induced metric depends on the task geometry and evaluation horizon.
\end{itemize}

The resulting picture is geometry-dependent. The behavior-induced metric improves the Mirror-Prox mean operator when it yields a more favorable saddle-point geometry than the covariance metric. We make this condition explicit through an exact spectral comparison and show through numerical mean-operator analysis and stochastic experiments that it holds in several standard prediction benchmarks.

\section{Background}

\subsection{Notation}

We consider a discounted Markov decision process
$(\mathcal{S},\mathcal{A},P,r,\gamma)$, where $\mathcal{S}$ is a finite state space, $\mathcal{A}$ is a finite action space, $P(s'|s,a)$ is the transition kernel, $r(s,a,s')$ is the reward, and $\gamma\in(0,1]$ is the discount factor. The target policy is denoted by $\pi$ and the behavior policy by $\mu$. The data are generated by $\mu$, while the value function of $\pi$ is to be estimated. Unless otherwise stated, all expectations are taken with respect to the stationary trajectory induced by the behavior policy $\mu$, with importance sampling used to correct the target-policy Bellman term.

For a policy $\nu\in\{\pi,\mu\}$, let $P_\nu\in\mathbb{R}^{|\mathcal{S}|\times |\mathcal{S}|}$ denote the state transition matrix induced by $\nu$:
\begin{equation}
[P_\nu]_{ss'}=\sum_{a\in\mathcal{A}}\nu(a|s)P(s'|s,a).
\end{equation}
Let $d_\mu$ be the stationary distribution of $P_\mu$, and let
\begin{equation}
D_\mu=\operatorname{diag}(d_\mu)
\end{equation}
be the corresponding diagonal weighting matrix. We use a linear value approximation
\begin{equation}
v_\theta(s)=\theta^\top\phi(s),
\end{equation}
where $\phi(s)\in\mathbb{R}^d$ is the feature vector and $\theta\in\mathbb{R}^d$ is the primary parameter. The feature matrix is
\begin{equation}
\Phi=\begin{pmatrix}\phi(s_1)^\top\\\cdots\\\phi(s_{|\mathcal{S}|})^\top\end{pmatrix}
\in\mathbb{R}^{|\mathcal{S}|\times d}.
\end{equation}
For compactness, we write $\phi_t=\phi(s_t)$ and $\phi_{t+1}=\phi(s_{t+1})$. When the next state is sampled according to the target-policy transition we write $\phi_{t+1}^\pi$; when it is sampled according to the behavior-policy transition we write $\phi_{t+1}^\mu$. In the sample-based off-policy updates below, $\phi_{t+1}$ is the observed next-state feature under $\mu$, and $\rho_t$ corrects the target-policy Bellman term.

The importance ratio is
\begin{equation}
\rho_t=\frac{\pi(a_t|s_t)}{\mu(a_t|s_t)}.
\end{equation}
The off-policy TD error is defined as
\begin{equation}
\delta_t=r_t+\gamma\theta_t^\top\phi_{t+1}-\theta_t^\top\phi_t.
\end{equation}

The key matrices used in the paper are
\begin{equation}
A_\pi=\mathbb{E}\left[\rho_t\phi_t(\phi_t-\gamma\phi_{t+1})^\top\right],\quad
b=\mathbb{E}\left[\rho_t r_t\phi_t\right].
\end{equation}
Equivalently, in matrix form,
\begin{equation}
A_\pi=\Phi^\top D_\mu(I-\gamma P_\pi)\Phi,
\quad
b=\Phi^\top D_\mu r_\pi,
\end{equation}
where $r_\pi(s)=\mathbb{E}_{a\sim\pi(\cdot|s),s'\sim P(\cdot|s,a)}[r(s,a,s')]$. The projected Bellman equation is $A_\pi\theta=b$. The auxiliary variable is denoted by $y\in\mathbb{R}^d$, and the joint variable used in the theoretical analysis is
\begin{equation}
z=\begin{pmatrix}\theta\\y\end{pmatrix}\in\mathbb{R}^{2d}.
\end{equation}

\subsection{Off-policy TD and saddle-point formulation}

Gradient-TD methods stabilize off-policy learning by introducing an auxiliary variable, but standard variants typically involve separate learning rates for the value and auxiliary variables \cite{sutton2008convergent,sutton2009fast}. Recent finite-sample and stochastic approximation studies further clarify the role of coupled recursions, two-timescale dynamics, and Markovian noise \cite{dalal2020finite,kaledin2020finite,doan2021finite}.

The saddle-point view starts from the objective
\begin{equation}
\min_{\theta}\max_y L(\theta,y)
=\langle b-A_\pi\theta,y\rangle-\frac{1}{2}\|y\|_M^2,
\label{eq:saddle_objective}
\end{equation}
where $M$ is a positive definite metric matrix. If $M$ is chosen properly, the solution still satisfies $A_\pi\theta=b$ while the optimization geometry changes. Proximal gradient TD and related methods exploit this formulation to obtain stable single-timescale updates \cite{liu2015finite,liu2018proximal}.

\section{Single-Timescale Hybrid TD}

Hybrid TD methods use behavior-policy information to alter the TD update direction. In this work we define the behavior-policy Bellman matrix
\begin{equation}
A_\mu=\mathbb{E}\left[\phi_t(\phi_t-\gamma\phi_{t+1}^\mu)^\top\right],
\end{equation}
where $\phi_{t+1}^\mu$ denotes the next-state feature induced by the behavior policy. Since $A_\mu$ need not be symmetric, the auxiliary metric is taken as
\begin{equation}
H=\frac{1}{2}(A_\mu+A_\mu^\top).
\end{equation}
The mean update of STHTD is
\begin{equation}
\begin{aligned}
y_{t+1}&=y_t+\alpha_t(b-A_\pi\theta_t-Hy_t),\\
\theta_{t+1}&=\theta_t+\alpha_t A_\pi^\top y_t.
\end{aligned}
\label{eq:mean_sthtd}
\end{equation}
Its sample-based off-policy form is
\begin{equation}
\begin{aligned}
\theta_{t+1}&=\theta_t+
\alpha_t\rho_t(\phi_t-\gamma\phi_{t+1})\phi_t^\top y_t,\\
y_{t+1}&=y_t+
\alpha_t\Big[(\rho_t\delta_t-\phi_t^\top y_t+\tfrac{1}{2}\gamma\phi_{t+1}^\top y_t)\phi_t
+\tfrac{1}{2}\gamma\phi_t^\top y_t\phi_{t+1}\Big].
\end{aligned}
\label{eq:sthtd}
\end{equation}
Compared with a covariance-metric auxiliary update, the STHTD auxiliary update contains additional hybrid terms involving $\phi_{t+1}$.

\section{Mirror-Prox Correction}

The saddle-point structure in Eq.~\eqref{eq:saddle_objective} induces a monotone operator
\begin{equation}
F(z)=
\begin{pmatrix}
-A_\pi^\top y\\
A_\pi\theta+Hy-b
\end{pmatrix},\quad z=(\theta,y).
\end{equation}
Mirror-Prox first evaluates the operator at the current point to form an intermediate prediction, and then uses the predicted point to correct the final update \cite{nemirovski2004prox,juditsky2011solving}. Applying this idea to Eq.~\eqref{eq:sthtd} gives the following update. Let
\begin{equation}
\begin{aligned}
\theta_t^m&=\theta_t+
\alpha_t\rho_t(\phi_t-\gamma\phi_{t+1})\phi_t^\top y_t,\\
y_t^m&=y_t+
\alpha_t\Big[(\rho_t\delta_t-\phi_t^\top y_t+\tfrac{1}{2}\gamma\phi_{t+1}^\top y_t)\phi_t
+\tfrac{1}{2}\gamma\phi_t^\top y_t\phi_{t+1}\Big].
\end{aligned}
\end{equation}
With the intermediate TD error
\begin{equation}
\delta_t^m=r_t+\gamma(\theta_t^m)^\top\phi_{t+1}-(\theta_t^m)^\top\phi_t,
\end{equation}
STHTD-MP performs
\begin{equation}
\begin{aligned}
\theta_{t+1}&=\theta_t+
\alpha_t\rho_t(\phi_t-\gamma\phi_{t+1})\phi_t^\top y_t^m,\\
y_{t+1}&=y_t+
\alpha_t\Big[(\rho_t\delta_t^m-\phi_t^\top y_t^m+\tfrac{1}{2}\gamma\phi_{t+1}^\top y_t^m)\phi_t
+\tfrac{1}{2}\gamma\phi_t^\top y_t^m\phi_{t+1}\Big].
\end{aligned}
\label{eq:sthtd_mp}
\end{equation}
The method uses one learning rate and roughly doubles the per-step first-order update cost, similarly to other extra-gradient methods.

\section{Theoretical Analysis}

This section formalizes the convergence argument for fixed-policy linear prediction. The setting is a fixed target policy with linear features, where the behavior-induced metric and the corresponding saddle-point mean dynamics can be analyzed explicitly.

\subsection{Assumptions and mean dynamics}

Let $\Phi\in\mathbb{R}^{|\mathcal{S}|\times d}$ be the feature matrix, $D_\mu=\operatorname{diag}(d_\mu)$ the stationary distribution matrix of the behavior policy, and $P_\mu$ the state transition matrix under $\mu$. Recall
\begin{equation}
A_\mu=\Phi^\top D_\mu(I-\gamma P_\mu)\Phi,
\quad
H=\frac{1}{2}(A_\mu+A_\mu^\top).
\end{equation}
The mean STHTD recursion in Eq.~\eqref{eq:mean_sthtd} is
\begin{equation}
z_{t+1}=z_t+\alpha_t(\mathcal{G}z_t+h),
\quad
z_t=\begin{pmatrix}\theta_t\\y_t\end{pmatrix},
\label{eq:joint_recursion}
\end{equation}
with
\begin{equation}
\mathcal{G}=\begin{pmatrix}0&A_\pi^\top\\-A_\pi&-H\end{pmatrix},
\quad
h=\begin{pmatrix}0\\b\end{pmatrix}.
\label{eq:joint_matrix}
\end{equation}

\begin{assumption}\label{ass:basic}
The following conditions hold: (i) $0<\gamma<1$; (ii) the Markov chain induced by $\mu$ is irreducible and has stationary distribution $d_\mu$ with $d_\mu(s)>0$ for all $s$; (iii) $d_\mu^\top P_\mu=d_\mu^\top$; (iv) $\Phi$ has full column rank; (v) $A_\pi$ is nonsingular; and (vi) features and rewards are uniformly bounded.
\end{assumption}

\begin{lemma}[Positive definiteness of the hybrid metric]\label{lem:h_spd}
Under Assumption~\ref{ass:basic}, $H=(A_\mu+A_\mu^\top)/2$ is positive definite.
\end{lemma}

\begin{proof}
For any nonzero $x\in\mathbb{R}^d$, let $v=\Phi x$. Since $\Phi$ has full column rank, $v\neq0$. We have
\begin{equation}
x^\top A_\mu x
=v^\top D_\mu(I-\gamma P_\mu)v
=\|v\|_{D_\mu}^2-\gamma\langle v,P_\mu v\rangle_{D_\mu}.
\end{equation}
By Cauchy--Schwarz,
\begin{equation}
\langle v,P_\mu v\rangle_{D_\mu}
\leq \|v\|_{D_\mu}\|P_\mu v\|_{D_\mu}.
\end{equation}
Jensen's inequality and stationarity of $d_\mu$ imply
\begin{equation}
\|P_\mu v\|_{D_\mu}^2
=\sum_s d_\mu(s)\left(\sum_{s'}P_\mu(s,s')v(s')\right)^2
\leq\sum_{s'}d_\mu(s')v(s')^2
=\|v\|_{D_\mu}^2.
\end{equation}
Therefore
\begin{equation}
x^\top A_\mu x\geq (1-\gamma)\|v\|_{D_\mu}^2>0.
\end{equation}
Since $x^\top Hx=x^\top A_\mu x$ for real $x$, $H$ is positive definite.
\end{proof}

\begin{theorem}[Hurwitz stability of the joint mean system]\label{thm:hurwitz}
Under Assumption~\ref{ass:basic}, the joint matrix $\mathcal{G}$ in Eq.~\eqref{eq:joint_matrix} is Hurwitz. Consequently, the ODE
\begin{equation}
\dot{z}(t)=\mathcal{G}z(t)+h
\label{eq:ode}
\end{equation}
has a unique globally exponentially stable equilibrium
\begin{equation}
z^*=\begin{pmatrix}\theta^*\\y^*\end{pmatrix}
=\begin{pmatrix}A_\pi^{-1}b\\0\end{pmatrix}.
\end{equation}
\end{theorem}

\begin{proof}
Let $\lambda$ be an eigenvalue of $\mathcal{G}$ with eigenvector $w=(u,v)\neq0$. Since
\begin{equation}
\mathcal{G}+\mathcal{G}^\top=\begin{pmatrix}0&0\\0&-2H\end{pmatrix},
\end{equation}
we obtain
\begin{equation}
2\operatorname{Re}(\lambda)\|w\|^2
=w^*(\mathcal{G}+\mathcal{G}^\top)w
=-2v^*Hv\leq0.
\end{equation}
Thus $\operatorname{Re}(\lambda)\leq0$. Suppose $\operatorname{Re}(\lambda)=0$. Since $H$ is positive definite by Lemma~\ref{lem:h_spd}, $v^*Hv=0$ implies $v=0$. The eigenvalue equations are
\begin{equation}
A_\pi^\top v=\lambda u,
\quad
-A_\pi u-Hv=\lambda v.
\end{equation}
With $v=0$, the first equation gives $\lambda u=0$. If $\lambda\neq0$, then $u=0$, contradicting $w\neq0$. If $\lambda=0$, the second equation gives $A_\pi u=0$, and nonsingularity of $A_\pi$ again gives $u=0$, a contradiction. Hence $\operatorname{Re}(\lambda)<0$ for every eigenvalue, so $\mathcal{G}$ is Hurwitz. The equilibrium follows from $A_\pi^\top y^*=0$ and $-A_\pi\theta^*-Hy^*+b=0$.
\end{proof}

\begin{assumption}\label{ass:sa}
The stochastic STHTD recursion can be written as
\begin{equation}
z_{t+1}=z_t+\alpha_t(\mathcal{G}z_t+h+\xi_{t+1}),
\end{equation}
where $\{\xi_{t+1}\}$ is a martingale-difference noise sequence with respect to the natural filtration $\{\mathcal{F}_t\}$, i.e., $\mathbb{E}[\xi_{t+1}|\mathcal{F}_t]=0$, and there exists $c_\xi>0$ such that
\begin{equation}
\mathbb{E}[\|\xi_{t+1}\|^2|\mathcal{F}_t]\leq c_\xi(1+\|z_t\|^2).
\end{equation}
The step sizes satisfy $\sum_t\alpha_t=\infty$ and $\sum_t\alpha_t^2<\infty$.
\end{assumption}

\begin{lemma}[Boundedness of the linear stochastic recursion]\label{lem:bounded}
Under Assumptions~\ref{ass:basic} and~\ref{ass:sa}, the iterates of the STHTD recursion are almost surely bounded.
\end{lemma}

\begin{proof}
By Theorem~\ref{thm:hurwitz}, $\mathcal{G}$ is Hurwitz. Hence for any positive definite matrix $Q$, the Lyapunov equation
\begin{equation}
\mathcal{G}^\top P+P\mathcal{G}=-Q
\label{eq:lyapunov_equation}
\end{equation}
has a unique positive definite solution $P$. Choose $Q=I$ and define the shifted error $e_t=z_t-z^*$, where $z^*$ is the unique equilibrium satisfying $\mathcal{G}z^*+h=0$. Then
\begin{equation}
e_{t+1}=e_t+\alpha_t(\mathcal{G}e_t+\xi_{t+1}).
\end{equation}
Let $V(e)=e^\top Pe$. Expanding $V(e_{t+1})$ gives
\begin{align}
V(e_{t+1})
&=V(e_t)+2\alpha_t e_t^\top P\mathcal{G}e_t
+2\alpha_t e_t^\top P\xi_{t+1}\nonumber\\
&\quad+\alpha_t^2(\mathcal{G}e_t+\xi_{t+1})^\top
P(\mathcal{G}e_t+\xi_{t+1}).
\end{align}
Taking conditional expectation, using $\mathbb{E}[\xi_{t+1}\mid\mathcal{F}_t]=0$, and expanding the quadratic term gives
\begin{align}
\mathbb{E}[V(e_{t+1})\mid\mathcal{F}_t]
&=V(e_t)+\alpha_t e_t^\top(P\mathcal{G}+\mathcal{G}^\top P)e_t\nonumber\\
&\quad+\alpha_t^2\,e_t^\top\mathcal{G}^\top P\mathcal{G}\,e_t
+\alpha_t^2\,\mathbb{E}[\xi_{t+1}^\top P\xi_{t+1}\mid\mathcal{F}_t]\nonumber\\
&\quad+2\alpha_t^2\,e_t^\top\mathcal{G}^\top P\,\mathbb{E}[\xi_{t+1}\mid\mathcal{F}_t]\\
&=V(e_t)+\alpha_t e_t^\top(P\mathcal{G}+\mathcal{G}^\top P)e_t\nonumber\\
&\quad+\alpha_t^2\,e_t^\top\mathcal{G}^\top P\mathcal{G}\,e_t
+\alpha_t^2\,\mathbb{E}[\xi_{t+1}^\top P\xi_{t+1}\mid\mathcal{F}_t],
\end{align}
where the cross term $2\alpha_t^2 e_t^\top\mathcal{G}^\top P\,\mathbb{E}[\xi_{t+1}\mid\mathcal{F}_t]$ vanishes because $\xi_{t+1}$ is a martingale difference. The deterministic quadratic term is bounded by $\|\mathcal{G}^\top P\mathcal{G}\|_2\|e_t\|^2$. For the noise term, the second-moment bound in Assumption~\ref{ass:sa} gives $\mathbb{E}[\|\xi_{t+1}\|^2\mid\mathcal{F}_t]\leq c_\xi(1+\|z_t\|^2)$, and boundedness of the equilibrium implies $1+\|z_t\|^2\leq c_0(1+\|e_t\|^2)$ for $c_0=2\max(1,\|z^*\|^2)$, so
\begin{equation}
\mathbb{E}[\xi_{t+1}^\top P\xi_{t+1}\mid\mathcal{F}_t]\leq M_P\,c_\xi\,c_0\,(1+\|e_t\|^2).
\end{equation}
Combining these bounds with Eq.~\eqref{eq:lyapunov_equation} yields
\begin{align}
\mathbb{E}[V(e_{t+1})\mid\mathcal{F}_t]
&\leq V(e_t)-\alpha_t\|e_t\|^2
+c_1\alpha_t^2(1+\|e_t\|^2),
\end{align}
with $c_1=\|\mathcal{G}^\top P\mathcal{G}\|_2+M_Pc_\xi c_0$. Since $P$ is positive definite, there exist constants $m_P,M_P>0$ such that $m_P\|e\|^2\leq V(e)\leq M_P\|e\|^2$, hence $\|e_t\|^2\geq V(e_t)/M_P$. Choosing $c=1/M_P$ and absorbing the $\alpha_t^2$ proportional term into $C\alpha_t^2$ once $\alpha_tc_1\leq c/2$ (which holds for all sufficiently large $t$ since $\alpha_t\to0$), we obtain
\begin{equation}
\mathbb{E}[V(e_{t+1})\mid\mathcal{F}_t]
\leq (1-c\alpha_t/2)V(e_t)+C\alpha_t^2.
\end{equation}
This is a standard Robbins--Siegmund recursion: combined with $\sum_t\alpha_t^2<\infty$ it implies that $V(e_t)$ converges almost surely to a finite random variable, and hence $\sup_t\|e_t\|<\infty$ almost surely~\cite[Chapter~3]{borkar2023stochastic}.
\end{proof}

\begin{theorem}[Almost-sure convergence of STHTD]\label{thm:sa_convergence}
Under Assumptions~\ref{ass:basic} and~\ref{ass:sa}, the STHTD iterates converge almost surely to the projected Bellman fixed point:
\begin{equation}
\theta_t\to A_\pi^{-1}b,
\quad
y_t\to0.
\end{equation}
\end{theorem}

\begin{proof}
The proof proceeds in three steps. First, Theorem~\ref{thm:hurwitz} shows that $\mathcal{G}$ is Hurwitz and that the ODE in Eq.~\eqref{eq:ode} has the unique globally exponentially stable equilibrium $z^*=(A_\pi^{-1}b,0)$. Second, Lemma~\ref{lem:bounded} establishes almost-sure boundedness of the stochastic iterates; this boundedness is derived from Hurwitz stability and is not assumed. Third, Assumption~\ref{ass:sa} gives the remaining stochastic approximation conditions: non-summable and square-summable step sizes and martingale-difference noise with controlled second moment. Therefore, the ODE method for stochastic approximation \cite{borkar2000ode,borkar2023stochastic} applies to the recursion, and every limit point of the stochastic process lies in the internally chain transitive set of the limiting ODE. Since the ODE has a unique globally asymptotically stable equilibrium, this set is the singleton $\{z^*\}$. Hence $z_t\to z^*$ almost surely, which gives $\theta_t\to A_\pi^{-1}b$ and $y_t\to0$.
\end{proof}

\subsection{Convergence-rate comparison}

We now compare the ergodic saddle-point gap of the single-step STHTD update and the Mirror-Prox update. This rate comparison uses the standard projected variational-inequality setting over compact convex sets $\Theta$ and $Y$, with fixed horizon-dependent step sizes and a stochastic oracle. It complements Theorem~\ref{thm:sa_convergence}, which establishes almost-sure convergence for the non-projected stochastic approximation recursion with diminishing step sizes.

Let $Z=\Theta\times Y$, $D=\sup_{z,z'\in Z}\|z-z'\|$, and define the monotone saddle-point operator
\begin{equation}
F(z)=
\begin{pmatrix}
-A_\pi^\top y\\
A_\pi\theta+Hy-b
\end{pmatrix}.
\label{eq:operator}
\end{equation}
This operator is the negative of the affine mean-update direction in Eq.~\eqref{eq:joint_recursion}, namely $F(z)=-(\mathcal{G}z+h)$.
For $\bar z_n=n^{-1}\sum_{t=1}^n z_t$, define the primal-dual gap
\begin{equation}
\operatorname{Gap}(\bar z_n)
=\max_{y\in Y}L(\bar\theta_n,y)
-\min_{\theta\in\Theta}L(\theta,\bar y_n).
\end{equation}
For the linear convex-concave saddle objective in Eq.~\eqref{eq:saddle_objective}, this saddle gap is controlled by the corresponding variational-inequality gap induced by $F$.

\begin{assumption}\label{ass:rate}
The operator $F$ is monotone and $L$-Lipschitz on $Z$. The stochastic oracle is i.i.d. across calls and satisfies $\mathbb{E}[\widehat F(z)|z]=F(z)$ and $\mathbb{E}\|\widehat F(z)-F(z)\|^2\leq\sigma^2$. Moreover, $\mathbb{E}\|\widehat F(z)\|^2\leq G^2$ on $Z$.
\end{assumption}

\begin{proposition}[Rate of projected STHTD]\label{prop:sthtd_rate}
Under Assumption~\ref{ass:rate}, the projected single-step STHTD update with constant step size $\alpha$ satisfies
\begin{equation}
\mathbb{E}[\operatorname{Gap}(\bar z_n)]
\leq
\frac{D^2}{2\alpha n}+\frac{\alpha G^2}{2}.
\label{eq:sthtd_rate}
\end{equation}
Choosing the horizon-dependent fixed step size $\alpha=D/(G\sqrt n)$ gives
\begin{equation}
\mathbb{E}[\operatorname{Gap}(\bar z_n)]
=O\left(\frac{DG}{\sqrt n}\right).
\end{equation}
\end{proposition}

\begin{proof}
The projected update is $z_{t+1}=\Pi_Z(z_t-\alpha\widehat F(z_t))$. Non-expansiveness of projection gives, for any $z\in Z$,
\begin{equation}
\|z_{t+1}-z\|^2
\leq\|z_t-z\|^2-2\alpha\langle \widehat F(z_t),z_t-z\rangle
+\alpha^2\|\widehat F(z_t)\|^2.
\end{equation}
Rearranging gives
\begin{equation}
\langle \widehat F(z_t),z_t-z\rangle
\leq \frac{\|z_t-z\|^2-\|z_{t+1}-z\|^2}{2\alpha}
+\frac{\alpha}{2}\|\widehat F(z_t)\|^2.
\end{equation}
Taking conditional expectation and using unbiasedness yields
\begin{equation}
\mathbb{E}\langle F(z_t),z_t-z\rangle
\leq \frac{\mathbb{E}\|z_t-z\|^2-\mathbb{E}\|z_{t+1}-z\|^2}{2\alpha}
+\frac{\alpha G^2}{2}.
\end{equation}
For a monotone variational inequality, $\operatorname{Gap}(z_t)\leq \max_{z\in Z}\langle F(z_t),z_t-z\rangle$. Summing from $t=1$ to $n$, using $\|z_t-z\|\leq D$, and applying convexity of the gap to the ergodic average gives Eq.~\eqref{eq:sthtd_rate}.
\end{proof}

\begin{proposition}[Rate of stochastic STHTD-MP]\label{prop:mp_rate}
Under the projected setting and the i.i.d. oracle condition in Assumption~\ref{ass:rate}, the projected stochastic Mirror-Prox version satisfies the standard bound
\begin{equation}
\mathbb{E}[\operatorname{Gap}(\bar z_n)]
\leq
C_1\frac{LD^2}{n}+C_2\frac{\sigma D}{\sqrt n},
\label{eq:mp_rate}
\end{equation}
for universal constants $C_1,C_2>0$ and an admissible step size of order $1/L$ for the deterministic component.
\end{proposition}

\begin{proof}
This is the standard stochastic Mirror-Prox bound for Lipschitz monotone variational inequalities with unbiased i.i.d. stochastic oracles \cite{nemirovski2004prox,juditsky2011solving}. The deterministic term follows from the extra-gradient correction, which cancels the leading rotation error of the operator and gives an $O(1/n)$ term. The second term is due to stochastic oracle noise; averaging independent oracle fluctuations yields the unavoidable $O(1/\sqrt n)$ contribution. Substituting the TD saddle-point operator in Eq.~\eqref{eq:operator} gives Eq.~\eqref{eq:mp_rate}.
\end{proof}

\begin{remark}\label{rem:iid_scope}
Proposition~\ref{prop:mp_rate} is stated under Assumption~\ref{ass:rate} with i.i.d. stochastic oracles. The behavior trajectories used in our experiments are Markovian, so the rate does not transfer directly; a finite-time bound under Markovian sampling requires an additional mixing-time penalty in the spirit of~\cite{dalal2020finite,kaledin2020finite,doan2021finite}. The exact mean-operator spectral analysis of Section~\ref{sec:numerical_analysis} is independent of the sampling model, and the steady-state AUC and step-size robustness results of Section~\ref{sec:robustness} are reported directly on Markovian trajectories.
\end{remark}

Propositions~\ref{prop:sthtd_rate} and~\ref{prop:mp_rate} explain the expected convergence-speed difference. For STHTD, the structural operator term and the stochastic term are both absorbed in an $O(n^{-1/2})$ bound. For STHTD-MP, the deterministic Lipschitz term is improved to $O(n^{-1})$, while the stochastic noise term remains $O(n^{-1/2})$.

\begin{corollary}[Speed advantage under a structure-dominant regime]\label{cor:speedup}
Consider a fixed horizon $n$ and suppose Assumption~\ref{ass:rate} holds for both STHTD and STHTD-MP. Let $G_{\rm STH}$ be the second-moment bound in Proposition~\ref{prop:sthtd_rate}, and let $(L,D,\sigma_{\rm MP})$ be the constants in Proposition~\ref{prop:mp_rate}. If
\begin{equation}
C_1\frac{LD}{\sqrt n}+C_2\sigma_{\rm MP}<G_{\rm STH},
\label{eq:speed_condition}
\end{equation}
then the theoretical upper bound on the expected ergodic gap of projected STHTD-MP is smaller than that of projected STHTD at horizon $n$.
This condition is a qualitative structural sufficient condition: the constants are usually not tightly identifiable from sample trajectories, and in practice we use AUC-based trajectory statistics and exact mean-operator spectral factors as empirical and numerical proxies.
\end{corollary}

\begin{proof}
By Proposition~\ref{prop:sthtd_rate}, projected STHTD with the optimized constant step size satisfies
\begin{equation}
\mathbb{E}[\operatorname{Gap}_{\rm STH}(\bar z_n)]
\leq O\left(\frac{DG_{\rm STH}}{\sqrt n}\right).
\end{equation}
By Proposition~\ref{prop:mp_rate}, projected STHTD-MP satisfies
\begin{equation}
\mathbb{E}[\operatorname{Gap}_{\rm MP}(\bar z_n)]
\leq
C_1\frac{LD^2}{n}+C_2\frac{\sigma_{\rm MP}D}{\sqrt n}
=\frac{D}{\sqrt n}\left(C_1\frac{LD}{\sqrt n}+C_2\sigma_{\rm MP}\right).
\end{equation}
Thus condition~\eqref{eq:speed_condition} is sufficient for the STHTD-MP bound to be smaller than the STHTD bound, up to the same universal-constant convention used in the two preceding propositions.
\end{proof}

Corollary~\ref{cor:speedup} clarifies the role of Mirror-Prox. STHTD-MP is favored when the deterministic coupling term represented by $L$ is large enough for the $O(n^{-1})$ improvement to matter, while the stochastic oracle noise $\sigma_{\rm MP}$ remains moderate. In the experiments, this structure is reflected by trajectory-level statistics such as AUC error and across-seed variability.

\subsection{Mean-operator comparison with GTD2-MP}

The preceding comparison explains why Mirror-Prox improves STHTD. We now compare STHTD-MP with the closer baseline GTD2-MP. Both methods use Mirror-Prox, but they use different saddle-point metrics. GTD2-MP corresponds to the covariance metric
\begin{equation}
C=\mathbb{E}_\mu[\phi_t\phi_t^\top]=\Phi^\top D_\mu\Phi,
\end{equation}
whereas STHTD-MP uses the behavior-induced symmetric metric
\begin{equation}
H=\frac{1}{2}(A_\mu+A_\mu^\top).
\end{equation}

The metric also determines the key matrix that appears after eliminating the auxiliary variable in the corresponding mean squared projected Bellman-error geometry:
\begin{equation}
B_M=A_\pi^\top M^{-1}A_\pi,\quad M\in\{C,H\}.
\label{eq:key_matrix}
\end{equation}
For $M=C$, $B_M$ is the GTD2 key matrix. For $M=H$, it is the behavior-induced, or hybrid, key matrix. A natural structural hypothesis behind the speed comparison is that the hybrid metric improves this key matrix relative to the covariance metric. The mean-rate result we use in Corollary~\ref{cor:gtd2mp_compare} below compares spectral radii of the exact Mirror-Prox mean recursion and does not require this hypothesis; we state it here because it is the geometric mechanism that explains the speedup whenever it holds, and Table~\ref{tab:key_matrix_check} verifies it on the benchmarks we study.

\begin{assumption}[Key-matrix improvement]\label{ass:key_matrix}
The matrices $C$ and $H$ are positive definite, $A_\pi$ is nonsingular, and the key matrices in Eq.~\eqref{eq:key_matrix} satisfy
\begin{equation}
\lambda_{\min}(B_H)>\lambda_{\min}(B_C),
\quad
\kappa(B_H)\leq\kappa(B_C),
\label{eq:key_matrix_condition}
\end{equation}
where $\kappa(B)=\lambda_{\max}(B)/\lambda_{\min}(B)$.
\end{assumption}

Under this condition, the idealized metric-preconditioned mean dynamics associated with $H$ has a better key-matrix geometry than the corresponding GTD2 dynamics associated with $C$. For a fixed positive definite key matrix $B$, the best constant-step-size linear factor of gradient descent on the associated quadratic is $(\kappa(B)-1)/(\kappa(B)+1)$. Hence a smaller key-matrix condition number directly indicates a better attainable mean linear factor, while a larger $\lambda_{\min}$ indicates stronger contraction for sufficiently small common step sizes. The two requirements on $\lambda_{\min}$ and $\kappa$ are not logically independent in general, since the smallest and largest eigenvalues can both shift under the metric change, but they hold jointly on the standard finite-state benchmarks reported in Table~\ref{tab:key_matrix_check}.

For a generic metric $M\in\{C,H\}$, define the linear saddle-point operator matrix
\begin{equation}
K_M=\begin{pmatrix}0&-A_\pi^\top\\A_\pi&M\end{pmatrix}.
\end{equation}
The deterministic Mirror-Prox error recursion for the mean operator is
\begin{equation}
e_{t+1}=R_M(\alpha)e_t,
\quad
R_M(\alpha)=I-\alpha K_M+\alpha^2K_M^2.
\label{eq:mp_error_matrix}
\end{equation}
Thus the exact asymptotic linear convergence factor is
\begin{equation}
q_M(\alpha)=\rho(R_M(\alpha)),
\label{eq:q_metric}
\end{equation}
where $\rho(\cdot)$ denotes the spectral radius.

\begin{corollary}[Exact mean-rate comparison with GTD2-MP]\label{cor:gtd2mp_compare}
Assume that both $K_C$ and $K_H$ are such that $q_C(\alpha_C)<1$ and $q_H(\alpha_H)<1$ for admissible step sizes $\alpha_C$ and $\alpha_H$. If
\begin{equation}
q_H(\alpha_H)<q_C(\alpha_C),
\label{eq:gtd2mp_speed_condition}
\end{equation}
then the deterministic mean STHTD-MP recursion converges linearly faster than the deterministic mean GTD2-MP recursion under these step sizes. In particular, if $\alpha_C^*$ and $\alpha_H^*$ minimize $q_C$ and $q_H$ on a common admissible grid and $q_H(\alpha_H^*)<q_C(\alpha_C^*)$, then STHTD-MP has a smaller best attainable mean linear convergence factor on that grid.
\end{corollary}

\begin{proof}
For a fixed metric $M$, Eq.~\eqref{eq:mp_error_matrix} gives $e_t=R_M(\alpha)^t e_0$. If $q_M(\alpha)<1$, then for any $\epsilon>0$ there exists a constant $c_\epsilon$ such that
\begin{equation}
\|e_t\|\leq c_\epsilon(q_M(\alpha)+\epsilon)^t\|e_0\|.
\end{equation}
Therefore the method with the smaller spectral radius has the smaller asymptotic linear factor. Applying this statement to $M=H$ and $M=C$ proves the result.
\end{proof}

Corollary~\ref{cor:gtd2mp_compare} is stated in terms of the exact mean spectral radius rather than a loose big-$O$ constant. This makes the comparison numerically verifiable for each finite benchmark and ties the predicted speed difference directly to the geometry of $A_\pi$, $C$, and $H$.

\section{Experiments}

\subsection{Protocol}

We evaluate on four standard off-policy prediction benchmarks: the two-state counterexample, Baird's counterexample, Random Walk, and Boyan Chain. The baselines are off-policy TD, GTD2, TDC, TDRC, GTD2-MP, HTD, ETD, STHTD, and STHTD-MP. GTD2, TDC, and HTD use separate value-function and auxiliary-variable learning rates, both of which are tuned. TDRC follows the original prediction-setting recommendation and uses a shared learning rate with regularization parameter fixed at $1.0$. We focus the main comparison on online first-order TD methods; batch least-squares methods such as LSTD are not included in the main learning-curve figures.

All four benchmarks are off-policy. In the two-state counterexample and Baird's counterexample, the target policy is deterministic while the behavior policy assigns positive probability to non-target actions, leading to strong importance-ratio mismatch. In Random Walk and Boyan Chain, the behavior policy chooses the two nonterminal actions with probability $0.5/0.5$, while the target policy uses $0.4/0.6$; these two benchmarks are therefore mildly off-policy. The Random Walk true values are computed from the target-policy Bellman equation with $\gamma=0.99$.

All algorithms are tuned on the same tuning seeds with a predefined grid. Reported results are evaluated on 100 disjoint random seeds. The tuning objective is the average prediction error over the last 20\% of the tuning trajectory. We report mean curves with shaded $\pm$ one-standard-deviation bands over the 100 evaluation runs. Tables report two scalar summaries: the steady-state AUC error, defined as the time-average RMSVE over the last 50\% of each trajectory, and the final error at the last step. The last-50\% AUC skips the early-phase transient and therefore evaluates each algorithm by its steady-state behavior; this is relevant for algorithms with high-variance early dynamics such as ETD, which converges with a high variance at the beginning on the two-state counterexample~\cite{chen2023mretrace}. Divergent runs are kept as $\infty$ or NaN and propagate to the reported means and standard deviations, without any clipping.

\subsection{Learning curves}

Figures~\ref{fig:two_state_curve}--\ref{fig:boyan_curve} show the learning curves. The proposed STHTD-MP is designed to change the Mirror-Prox TD geometry relative to GTD2-MP, and the curves should be interpreted together with the exact mean-operator analysis in Section~\ref{sec:numerical_analysis}. Empirically, STHTD-MP is consistently more stable than the underlying STHTD update in the counterexample settings and reaches low final error on the two longer-horizon non-counterexample tasks. The results also show that well-tuned two-learning-rate baselines, TDRC, ETD, and simple TD-style methods can be very strong in mildly off-policy prediction tasks.

\begin{figure}[!t]
\centering
\includegraphics[width=0.95\textwidth]{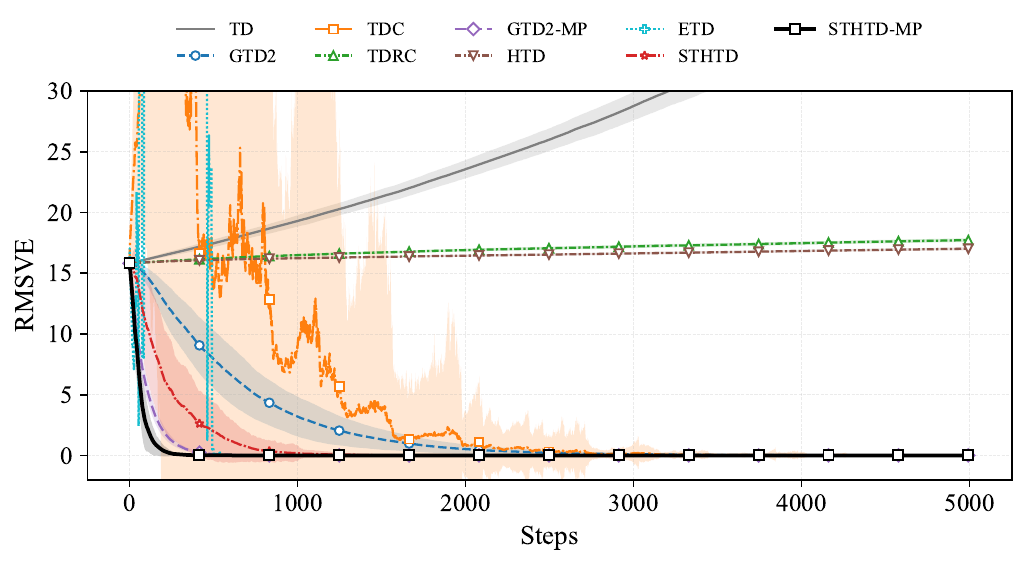}
\caption{Prediction learning curves on the two-state counterexample. Curves show means over 100 independent runs; shaded regions show $\pm$ one sample standard deviation. Stochastic single-step ETD exhibits early-phase high-variance transients before converging, which is consistent with the observation of Chen et al.~\cite{chen2023mretrace} that emphatic TD ``converges with a high variance at the beginning in the 2-state counterexample.''}
\label{fig:two_state_curve}
\end{figure}

\begin{figure}[!t]
\centering
\includegraphics[width=0.95\textwidth]{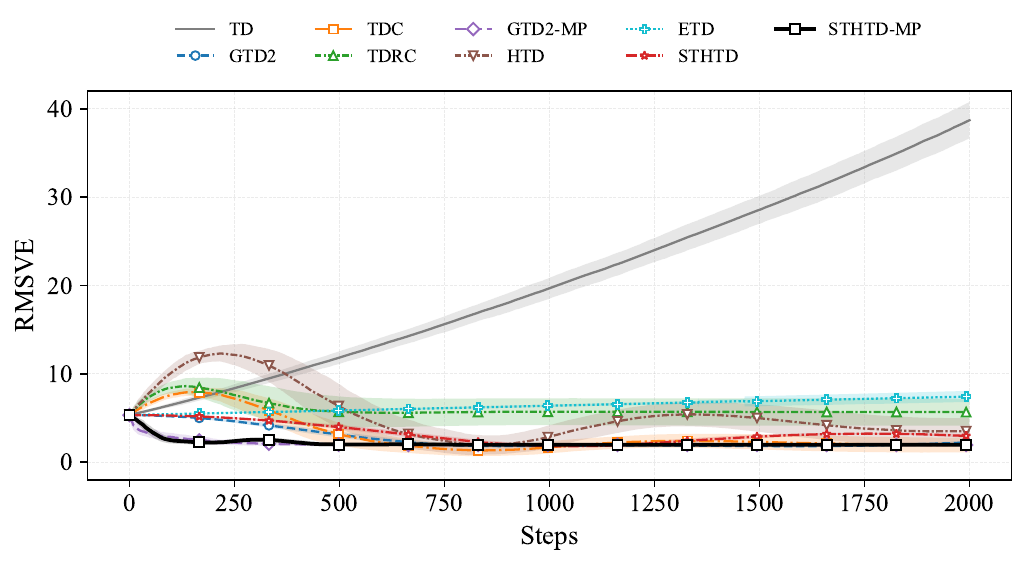}
\caption{Prediction learning curves on Baird's counterexample. Curves show means over 100 independent runs, and shaded regions show $\pm$ one sample standard deviation.}
\label{fig:baird_curve}
\end{figure}

\begin{figure}[!t]
\centering
\includegraphics[width=0.95\textwidth]{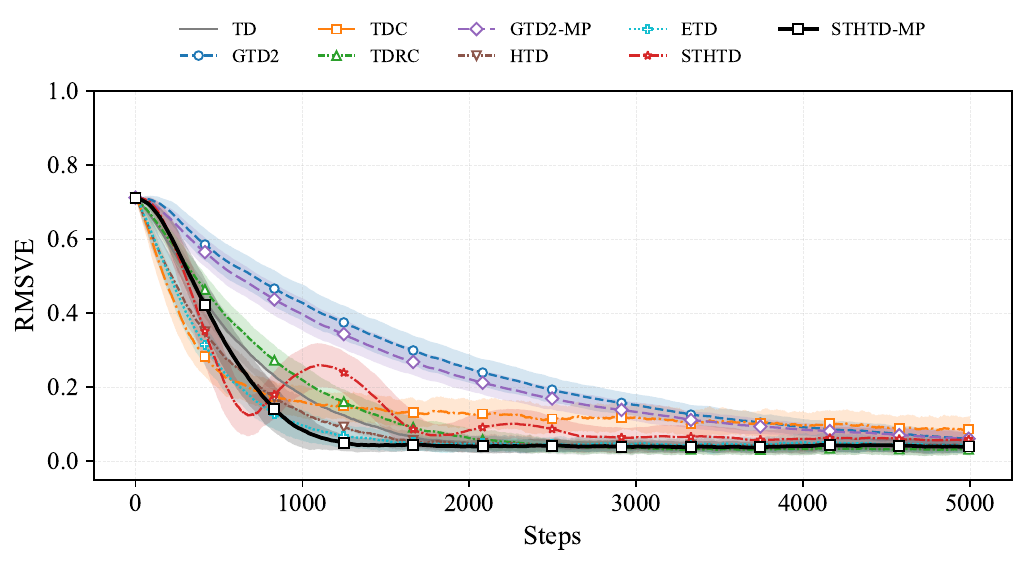}
\caption{Prediction learning curves on Random Walk. Curves show means over 100 independent runs, and shaded regions show $\pm$ one sample standard deviation.}
\label{fig:random_walk_curve}
\end{figure}

\begin{figure}[!t]
\centering
\includegraphics[width=0.95\textwidth]{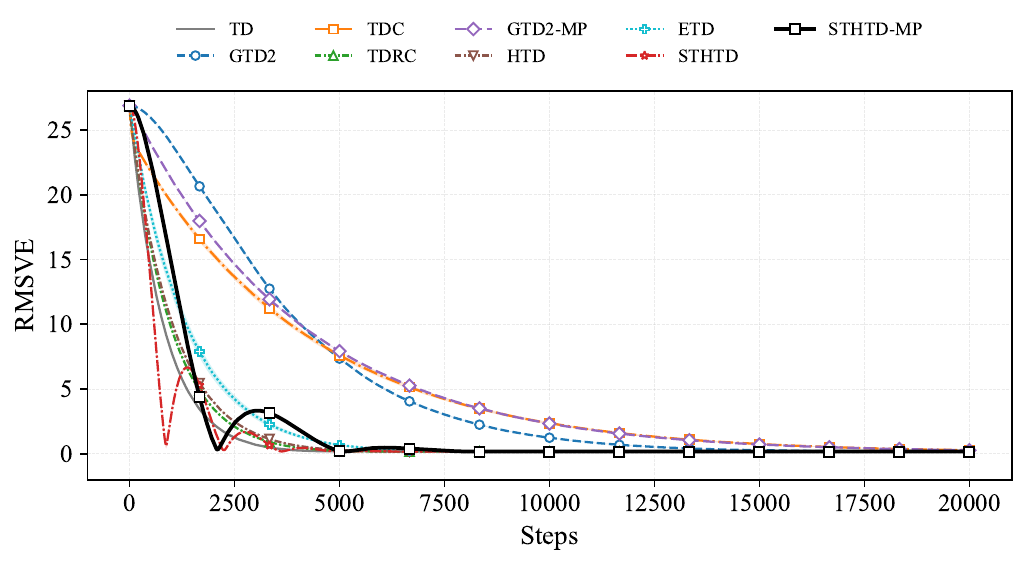}
\caption{Prediction learning curves on Boyan Chain. Curves show means over 100 independent runs; shaded regions show $\pm$ one sample standard deviation. The shaded band appears thin because the across-seed standard deviation is on the order of $10^{-2}$ while the ordinate spans tens.}
\label{fig:boyan_curve}
\end{figure}

On Boyan Chain, TD, TDRC, HTD, ETD, and STHTD-MP all converge to the same linear projected fixed point at RMSVE $\approx 0.167$, so their late-trajectory curves overlap in Figure~\ref{fig:boyan_curve}. The visible oscillation of STHTD is along the time axis of the mean curve, not across-seed dispersion. GTD2-MP and TDC converge more slowly within the $20{,}000$-step budget, which is consistent with the known slow convergence of TDC-style updates on Boyan Chain~\cite{maei2010gq}.

\subsection{Quantitative comparison}

Table~\ref{tab:auc} reports the steady-state AUC error, defined as the time-average RMSVE over the last 50\% of each trajectory. This metric captures the steady-state behavior of each algorithm rather than the early-phase transient and is therefore a sharper measure of how fast each algorithm enters its asymptotic regime. STHTD-MP attains the lowest steady-state AUC among the proposed hybrid variants on the two-state and Baird counterexamples, while TD-style and regularized-correction baselines remain strong on the mildly off-policy tasks. These results support a geometry-dependent claim: replacing the covariance metric with the behavior-induced metric can improve the Mirror-Prox mean dynamics, while finite-sample performance also depends on sampling variance, feature conditioning, step-size tuning, and task difficulty.

The Random Walk results illustrate how the geometric advantage interacts with sampling noise. STHTD-MP improves on GTD2-MP by roughly $2.5\times$ in steady-state AUC (Table~\ref{tab:auc}), and the exact mean-operator analysis in Section~\ref{sec:numerical_analysis} shows that the hybrid Mirror-Prox factor $q_H(\alpha_H^*)$ is smaller than $q_C(\alpha_C^*)$ on this benchmark (Table~\ref{tab:exact_rate_gtd2mp}). At the same time, semi-gradient TD, TDRC, and HTD all reach a slightly lower or comparable AUC. Random Walk is only mildly off-policy, with importance ratios $\rho\in[0.8,1.2]$, so the projected Bellman fixed point is stable even for plain semi-gradient TD and the dominant source of finite-sample error is the importance-weighted noise variance rather than the Mirror-Prox structural error. The geometric advantage of the hybrid metric, which manifests in the deterministic operator factor $q$, is therefore partially absorbed by single-update sampling noise of the auxiliary-variable recursion, which scales with $\|\phi\|^2$ rather than with the projected Bellman geometry. On the two-state counterexample the off-policy mismatch is large and the deterministic mean-operator factor dominates the finite-sample behavior, and STHTD-MP improves on GTD2-MP by nine orders of magnitude.

\begin{table}[!t]
\centering
\caption{Steady-state AUC error (last-50\% time-average RMSVE) over 100 independent runs, mean $\pm$ sample standard deviation. Lower is better. Bold marks the best online first-order TD method per environment.}
\label{tab:auc}
\resizebox{\textwidth}{!}{%
\begin{tabular}{lcccc}
\toprule
Algorithm & Two-state & Baird & Random Walk & Boyan Chain \\
\midrule
TD & $33.734\pm1.426$ & $28.757\pm1.512$ & $\mathbf{0.0336\pm0.0067}$ & $0.1680\pm0.0015$ \\
GTD2 & $4.79\times10^{-2}\pm5.03\times10^{-2}$ & $1.887\pm0.098$ & $0.1119\pm0.0211$ & $0.4001\pm0.0091$ \\
TDC & $6.88\times10^{-2}\pm2.54\times10^{-1}$ & $2.130\pm0.786$ & $0.1027\pm0.0140$ & $0.9298\pm0.0374$ \\
TDRC & $17.408\pm0.122$ & $5.672\pm1.534$ & $0.0337\pm0.0070$ & $0.1670\pm0.0012$ \\
GTD2-MP & $3.67\times10^{-12}\pm2.80\times10^{-12}$ & $\mathbf{1.933\pm0.011}$ & $0.1013\pm0.0176$ & $0.9017\pm0.0103$ \\
HTD & $16.783\pm0.085$ & $4.341\pm1.339$ & $0.0431\pm0.0064$ & $0.1672\pm0.0013$ \\
ETD & $\mathbf{9.68\times10^{-80}\pm9.68\times10^{-79}}$ & $6.903\pm0.349$ & $0.0462\pm0.0084$ & $\mathbf{0.1669\pm0.0017}$ \\
STHTD & $5.19\times10^{-6}\pm3.71\times10^{-5}$ & $2.660\pm0.105$ & $0.0630\pm0.0149$ & $0.1761\pm0.0032$ \\
STHTD-MP & $6.71\times10^{-21}\pm5.13\times10^{-20}$ & $1.946\pm0.010$ & $0.0401\pm0.0080$ & $0.1692\pm0.0026$ \\
\bottomrule
\end{tabular}}
\end{table}

Table~\ref{tab:final} reports the final error at the last step. Final error favors methods that eventually converge accurately, while the steady-state AUC reflects how quickly the algorithm reaches that regime. Values at the $10^{-12}$ level or below in Tables~\ref{tab:auc} and~\ref{tab:final} are at machine precision and indicate numerical convergence to zero; the explicit values are retained for reproducibility.

\begin{table}[!t]
\centering
\caption{Final prediction error (RMSVE at the last step) over 100 independent runs, mean $\pm$ sample standard deviation. Lower is better. Bold marks the best online first-order TD method per environment.}
\label{tab:final}
\resizebox{\textwidth}{!}{%
\begin{tabular}{lcccc}
\toprule
Algorithm & Two-state & Baird & Random Walk & Boyan Chain \\
\midrule
TD & $42.792\pm2.100$ & $38.704\pm2.066$ & $0.0319\pm0.0161$ & $0.1677\pm0.0051$ \\
GTD2 & $2.69\times10^{-3}\pm5.71\times10^{-3}$ & $2.213\pm0.184$ & $0.0610\pm0.0197$ & $0.1726\pm0.0054$ \\
TDC & $1.36\times10^{-3}\pm6.89\times10^{-3}$ & $\mathbf{1.915\pm0.832}$ & $0.0879\pm0.0333$ & $0.2777\pm0.0205$ \\
TDRC & $17.731\pm0.156$ & $5.656\pm1.533$ & $\mathbf{0.0312\pm0.0153}$ & $\mathbf{0.1667\pm0.0044}$ \\
GTD2-MP & $6.66\times10^{-22}\pm1.07\times10^{-21}$ & $1.932\pm0.011$ & $0.0605\pm0.0181$ & $0.2747\pm0.0116$ \\
HTD & $17.021\pm0.116$ & $3.507\pm1.461$ & $0.0419\pm0.0200$ & $0.1669\pm0.0042$ \\
ETD & $\mathbf{0.000\pm0.000}$ & $7.408\pm0.645$ & $0.0435\pm0.0204$ & $0.1669\pm0.0054$ \\
STHTD & $6.53\times10^{-9}\pm6.43\times10^{-8}$ & $2.942\pm0.212$ & $0.0586\pm0.0264$ & $0.1770\pm0.0141$ \\
STHTD-MP & $3.97\times10^{-40}\pm2.09\times10^{-39}$ & $1.944\pm0.010$ & $0.0394\pm0.0162$ & $0.1687\pm0.0066$ \\
\bottomrule
\end{tabular}}
\end{table}

\subsection{Step-size robustness}\label{sec:robustness}

The above comparison reports each algorithm at its tuned step size. A practitioner, however, rarely has the budget to tune $\alpha$ as carefully as we do here. We therefore evaluate \emph{step-size robustness}: for each algorithm and each environment we fix the auxiliary-variable step size at $\beta=0.05$ and the TDRC regularizer at $1.0$, and sweep the primary step size $\alpha\in\{10^{-4}, 3\times10^{-4}, 10^{-3}, 3\times10^{-3}, 10^{-2}, 3\times10^{-2}, 10^{-1}\}$. Each $(\alpha, \text{algorithm}, \text{environment})$ cell is evaluated on 30 independent seeds, and we plot the resulting steady-state AUC (last-50\% time-average RMSVE) against $\alpha$ on log--log axes. A robust algorithm has both a low minimum AUC over the grid and a wide region of $\alpha$ for which the AUC stays close to that minimum. This is the same sensitivity protocol used by Chen et al.~\cite{chen2023mretrace} on the same counterexamples.

\begin{figure}[!t]
\centering
\begin{subfigure}[t]{0.49\textwidth}
\centering
\includegraphics[width=\textwidth]{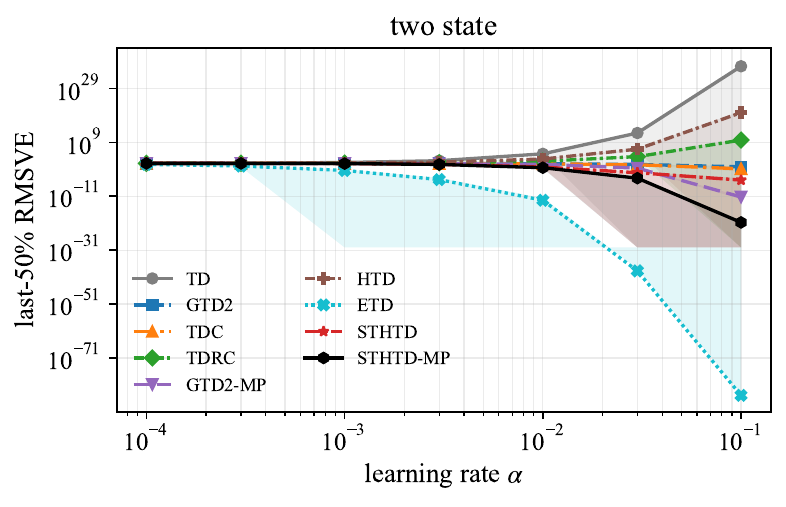}
\caption{Two-state counterexample.}
\label{fig:robustness_two_state}
\end{subfigure}\hfill
\begin{subfigure}[t]{0.49\textwidth}
\centering
\includegraphics[width=\textwidth]{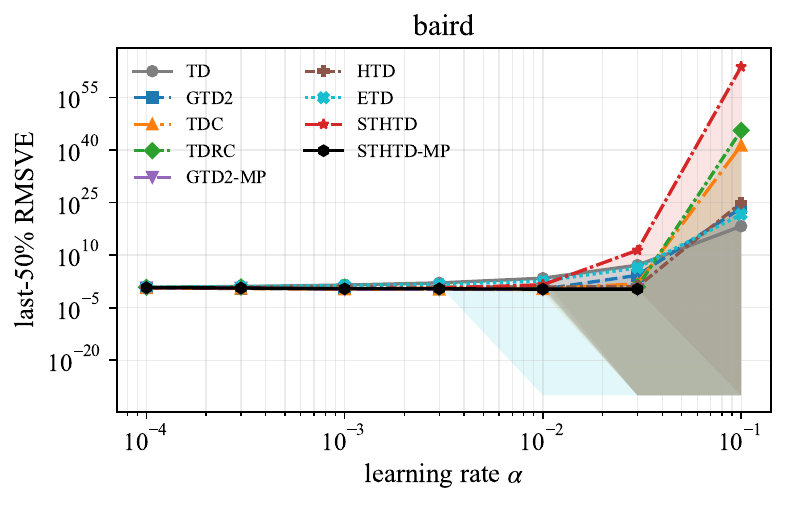}
\caption{Baird's counterexample.}
\label{fig:robustness_baird}
\end{subfigure}

\vspace{0.5em}

\begin{subfigure}[t]{0.49\textwidth}
\centering
\includegraphics[width=\textwidth]{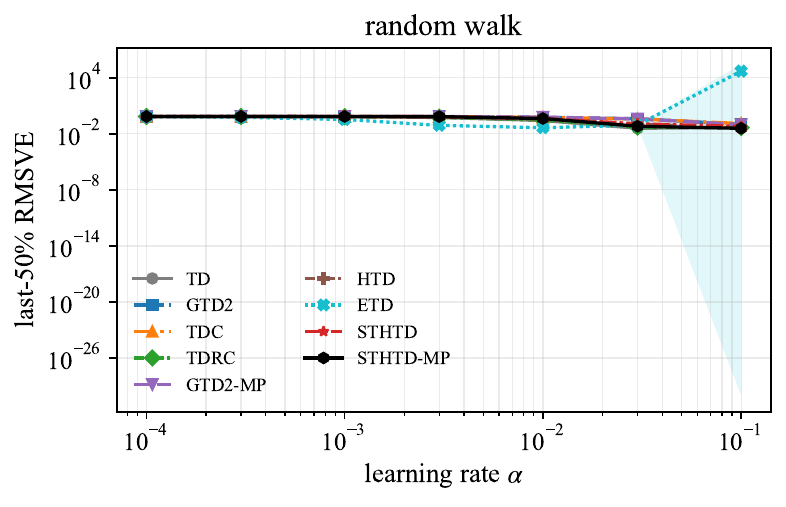}
\caption{Random Walk.}
\label{fig:robustness_random_walk}
\end{subfigure}\hfill
\begin{subfigure}[t]{0.49\textwidth}
\centering
\includegraphics[width=\textwidth]{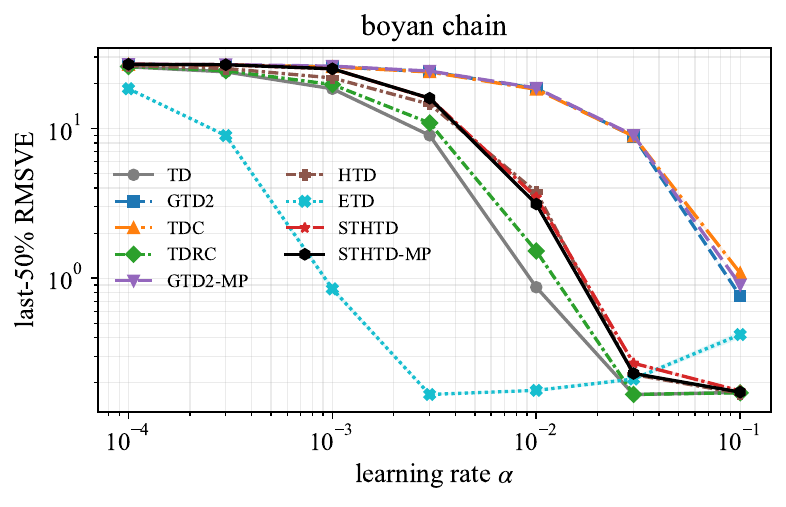}
\caption{Boyan Chain.}
\label{fig:robustness_boyan_chain}
\end{subfigure}
\caption{Step-size robustness. Each panel plots the steady-state AUC (last-50\% time-average RMSVE) versus the primary step size $\alpha$ over 30 independent seeds; shaded regions show $\pm$ one sample standard deviation. Both axes are logarithmic. A lower curve indicates better performance, and a flatter curve indicates greater robustness to step-size mis-specification.}
\label{fig:robustness}
\end{figure}

Figure~\ref{fig:robustness} reveals three patterns. \emph{First}, on the two-state counterexample, STHTD-MP is monotone in $\alpha$ and reaches $\sim2\times10^{-21}$ at $\alpha=10^{-1}$, beating GTD2-MP by roughly nine orders of magnitude across the entire grid. ETD has lower AUC than both at the largest tested $\alpha$ but exhibits high variance at small $\alpha$, consistent with the early-phase transient discussed earlier. \emph{Second}, on Baird's counterexample, STHTD-MP and GTD2-MP both bottom out near $1.94$ around $\alpha=10^{-2}$ and both fail at $\alpha=10^{-1}$ (with 23 and 15 out of 30 seeds diverging respectively); ETD diverges monotonically as $\alpha$ grows, which is consistent with the well-known observation that emphatic TD diverges on Baird~\cite{chen2023mretrace}. \emph{Third}, on Random Walk and Boyan Chain, STHTD-MP attains its minimum at $\alpha=10^{-1}$ and matches or beats the strong baselines TD, TDRC, and HTD there, while remaining bounded at every $\alpha$ where ETD blows up (e.g., Random Walk at $\alpha=10^{-1}$). The breadth of the low-AUC region for STHTD-MP is comparable to or wider than that of GTD2-MP on every benchmark, supporting the geometric claim that the behavior-induced metric yields a better-conditioned Mirror-Prox mean dynamics.

\section{Numerical Analysis}\label{sec:numerical_analysis}

The theoretical analysis identifies two mechanisms: Mirror-Prox reduces the deterministic structure error of the hybrid saddle-point operator, and the behavior-induced metric can improve the key matrix that governs the mean dynamics. The stochastic experiments verify the finite-sample behavior of the algorithms. This section connects the two by computing the exact finite-state matrices behind the benchmarks. In this way, the theoretical assumptions, the observed learning curves, and the numerical matrix properties form a single closed loop.

\subsection{Exact mean-rate comparison with GTD2-MP}

To test the key-matrix hypothesis and Corollary~\ref{cor:gtd2mp_compare}, we compute the exact finite-state matrices $A_\pi$, $C$, $A_\mu$, $H$, $K_C$, and $K_H$ for the four benchmarks. This is not a Monte Carlo simulation: the transition probabilities, stationary behavior distribution, and feature matrices are used directly to form the mean operators.

Table~\ref{tab:key_matrix_check} first verifies the structural hypothesis in Assumption~\ref{ass:key_matrix}. On the two-state, Random Walk, and Boyan Chain benchmarks, the hybrid key matrix $B_H=A_\pi^\top H^{-1}A_\pi$ has a larger smallest eigenvalue than $B_C=A_\pi^\top C^{-1}A_\pi$ and no worse conditioning. The improvement is especially clear on Random Walk and Boyan Chain, where the condition number decreases from $181.49$ to $18.71$ and from $63.14$ to $11.23$, respectively. Baird's counterexample does not satisfy the strict hypothesis because $A_\pi$ and the metric matrices are numerically singular in the over-parameterized feature representation.

\begin{table}[!t]
\centering
\caption{Verification of the key-matrix condition. Here $B_C=A_\pi^\top C^{-1}A_\pi$ is the GTD2 key matrix and $B_H=A_\pi^\top H^{-1}A_\pi$ is the hybrid key matrix. The factor is $(\kappa-1)/(\kappa+1)$, the best linear factor for the corresponding idealized quadratic gradient dynamics. Lower condition number and lower factor are better.}
\label{tab:key_matrix_check}
\resizebox{\textwidth}{!}{%
\begin{tabular}{lccccc}
\toprule
Environment & $\lambda_{\min}(B_C)$ / $\lambda_{\min}(B_H)$ & $\kappa(B_C)$ / $\kappa(B_H)$ & Factor $(C/H)$ & Condition & Interpretation \\
\midrule
Two-state & $0.0160$ / $0.0842$ & $1.00$ / $1.00$ & $0.000$ / $0.000$ & yes & stronger hybrid contraction \\
Baird & $\approx0$ / $\approx0$ & singular / singular & $1.000$ / $1.000$ & no & singular boundary case \\
Random Walk & $0.00475$ / $0.0275$ & $181.49$ / $18.71$ & $0.989$ / $0.899$ & yes & better conditioned hybrid key matrix \\
Boyan Chain & $0.00250$ / $0.0248$ & $63.14$ / $11.23$ & $0.969$ / $0.836$ & yes & better conditioned hybrid key matrix \\
\bottomrule
\end{tabular}}
\end{table}

After verifying this structural condition, we compute the Mirror-Prox spectral radii $q_C(\alpha)$ and $q_H(\alpha)$ in Eq.~\eqref{eq:q_metric}, both at the tuned learning rates used in the experiments and at the best learning rates over a common logarithmic grid. Table~\ref{tab:exact_rate_gtd2mp} gives these exact numerical results. On the two-state problem, the hybrid metric reduces the operator norm by 78.2\%, increases the Hurwitz margin from 0.0161 to 0.1094, and improves the best mean Mirror-Prox contraction factor from 0.9936 to 0.9026. On Boyan Chain, the operator norm is reduced by 17.3\%, the Hurwitz margin is improved by nearly five times, and the best contraction factor improves from 0.9975 to 0.9875. On Random Walk with $\gamma=0.99$, $\|K_H\|_2$ is larger than $\|K_C\|_2$, but the key matrix and Hurwitz margin are both more favorable, and the exact spectral radius still favors STHTD-MP. Therefore, for the three benchmarks that satisfy the key-matrix condition, the exact mean-operator analysis supports the claim that the proposed hybrid Mirror-Prox metric has a faster deterministic mean convergence factor than GTD2-MP.

Baird's counterexample has a different mean geometry. Its feature matrix has more features than states, and $A_\pi$ is numerically singular in this representation. The key-matrix condition fails and the exact spectral factors are essentially one for both metrics, so the deterministic mean-rate comparison is inconclusive. This matches the empirical behavior: Baird is an extreme off-policy case in which the behavior-induced metric does not yield the clear contraction advantage observed in the other three benchmarks.

\begin{table}[!t]
\centering
\caption{Exact deterministic mean-rate comparison between GTD2-MP and STHTD-MP. Here $C=\mathbb{E}_\mu[\phi\phi^\top]$ is the GTD2-MP metric and $H=\operatorname{sym}(A_\mu)$ is the STHTD-MP metric. $q(\alpha)=\rho(I-\alpha K+\alpha^2K^2)$ is the exact Mirror-Prox mean contraction factor. Lower $q$ is better.}
\label{tab:exact_rate_gtd2mp}
\resizebox{\textwidth}{!}{%
\begin{tabular}{lcccccc}
\toprule
Environment & $\|K_C\|_2$ & $\|K_H\|_2$ & Hurwitz margin $(C/H)$ & Tuned $q_C/q_H$ & Best $q_C/q_H$ & Conclusion \\
\midrule
Two-state & 2.516 & 0.548 & 0.0161 / 0.1094 & 0.9968 / 0.9786 & 0.9936 / 0.9026 & STHTD-MP faster \\
Baird & 2.912 & 1.796 & $\approx0$ / $\approx0$ & 1.0000 / 1.0000 & 1.0000 / 1.0000 & singular; inconclusive \\
Random Walk & 0.665 & 0.791 & 0.00489 / 0.0147 & 0.9990 / 0.9970 & 0.9951 / 0.9851 & STHTD-MP faster \\
Boyan Chain & 0.285 & 0.236 & 0.00252 / 0.0124 & 0.9995 / 0.9975 & 0.9975 / 0.9875 & STHTD-MP faster \\
\bottomrule
\end{tabular}}
\end{table}

\subsection{Why does STHTD struggle on Baird's counterexample?}

Baird's counterexample is an extreme off-policy setting. In our implementation, the target policy always selects the solid action, while the behavior policy selects it with probability $1/7$. Consequently, the importance ratio is $7$ for the solid action and $0$ for the dashed action. The $\theta$ block of STHTD is driven by this target-policy, importance-weighted direction, but the auxiliary block also contains behavior-induced hybrid terms involving $\phi_{t+1}$. In Baird's geometry, these two effects can be poorly aligned, producing oscillation and poor stability.

This diagnosis is also consistent with the exact mean-operator analysis in Table~\ref{tab:exact_rate_gtd2mp}. For Baird, both $C$ and $H$ lead to nearly zero Hurwitz margins because $A_\pi$ is numerically singular in the over-parameterized feature representation. The resulting mean-rate comparison is therefore inconclusive. The observed behavior reflects this singular off-policy geometry, where the behavior-induced metric has limited room to improve the covariance-metric dynamics.

\section{Related Work}

The instability of off-policy TD with function approximation was established in early theoretical and empirical studies \cite{baird1995residual,tsitsiklis1997analysis}. Importance sampling enables off-policy prediction but can introduce high-variance updates, especially under strong target-behavior mismatch \cite{precup2000eligibility}. Gradient-TD methods address the stability issue by introducing an auxiliary variable and optimizing projected Bellman-error objectives \cite{sutton2008convergent,sutton2009fast,maei2010gq}. Related control-oriented extensions show how these ideas connect to off-policy action-value learning, but they also introduce additional complications beyond fixed-policy prediction \cite{szepesvari2010toward}.

Two aspects of gradient-TD are especially relevant here. First, classical GTD2 and TDC use coupled primal and auxiliary recursions whose empirical behavior depends on relative step-size tuning. Finite-sample and stochastic-approximation analyses have clarified how two-timescale structure and Markovian noise affect such recursions \cite{dalal2020finite,kaledin2020finite,doan2021finite}. Second, the auxiliary metric changes the geometry of the saddle-point operator. Hybrid TD methods were motivated by faster learning through a mixture of TD and gradient-TD directions \cite{hackman2012faster}. Our work keeps this behavior-policy geometric ingredient but formulates it as a symmetric positive definite metric in a single-timescale saddle-point system.

Proximal gradient TD and related saddle-point formulations provide a natural route to single-timescale algorithms \cite{liu2015finite,liu2018proximal}. Mirror-Prox and stochastic extra-gradient methods provide a general mechanism for smooth monotone variational inequalities and convex-concave saddle-point problems \cite{nemirovski2004prox,juditsky2011solving}. Existing Mirror-Prox TD-style methods typically inherit the GTD2 covariance metric $C=\mathbb{E}_\mu[\phi\phi^\top]$. The distinction of STHTD-MP is that the extra-gradient step is applied after changing the metric to $H=\operatorname{sym}(A_\mu)$. Thus the contribution is not simply adding Mirror-Prox to TD, but changing the deterministic Mirror-Prox mean operator through behavior-induced geometry.

A line of recent work has further sharpened the analysis of single- and two-timescale gradient TD. Doan~\cite{doan2021finite} gives finite-time bounds for linear two-timescale stochastic approximation with restarting, while Dalal et al.~\cite{dalal2020finite} and Kaledin et al.~\cite{kaledin2020finite} establish finite-time and Markovian-noise analyses that apply to TDC-type recursions. TDRC introduces a regularized correction with a shared learning rate~\cite{ghiassian2020gradient}, which can be viewed as a single-timescale stabilization of TDC and serves as a direct baseline in our experiments. Hybrid TD methods~\cite{hackman2012faster} also exploit behavior-policy information, but as a mixture of TD and gradient-TD directions rather than as a positive definite metric, so the saddle-point reading of the speedup is absent. Relative to single-timescale proximal TD methods~\cite{liu2018proximal}, the distinguishing feature of STHTD-MP is the use of the symmetric part $H=\operatorname{sym}(A_\mu)$ of the behavior-induced Bellman matrix as the saddle-point metric, and the resulting key-matrix improvement made measurable in Section~\ref{sec:numerical_analysis}.

Recent reinforcement-learning work has also studied target networks, off-policy evaluation, and variance-reduced Markovian stochastic approximation \cite{uehara2020minimax,zhang2021breaking,wai2021variance}. These directions are complementary. Target-network and variance-reduction methods mainly control instability or sampling noise, whereas this paper isolates the effect of the saddle-point metric on linear off-policy prediction. The exact mean-operator comparison with GTD2-MP makes this geometric difference measurable in addition to the empirical learning curves.

\section{Discussion}

The present analysis focuses on fixed-policy off-policy prediction with linear function approximation. This setting exposes the role of the behavior-induced metric cleanly because $A_\mu$, $H$, and the saddle-point operator are fixed by the behavior policy and the feature map. Extending the same idea to control would require coupling the metric with changing target policies and bootstrapped action-value maximization. Nonlinear and deep function approximation would add another layer of dependence because the metric and the saddle-point operator would vary with the representation parameters.

The convergence and rate analyses describe two complementary views of the algorithm. The stochastic-approximation theorem establishes almost-sure stability for the non-projected linear recursion with diminishing step sizes. The projected variational-inequality bounds describe the structural effect of Mirror-Prox under an unbiased i.i.d. oracle model. The experiments use Markovian trajectories, so the theoretical rate results are paired with exact finite-state mean-operator calculations and trajectory-level statistics.

The exact GTD2-MP comparison is a deterministic mean-operator comparison for finite benchmarks. It gives a verifiable contraction factor and identifies when the behavior-induced metric improves the Mirror-Prox geometry. In finite samples, this mean effect interacts with variance, feature conditioning, and the strength of the off-policy mismatch, which explains why the empirical comparisons are reported with standard deviations across independent runs.

\section{Conclusion}

We presented STHTD-MP, a behavior-induced Mirror-Prox temporal-difference method for off-policy prediction with linear function approximation. Unlike GTD2-MP, which uses the feature covariance metric, STHTD-MP uses the symmetric part of the behavior-policy Bellman matrix as the saddle-point metric. This changes the deterministic Mirror-Prox mean operator rather than merely adding an extra-gradient step to an existing TD method.

The theoretical analysis establishes positive definiteness of the behavior-induced metric under standard finite-state assumptions, Hurwitz stability of the joint mean system, stochastic-approximation convergence, and ergodic gap bounds. More importantly, the exact mean-rate comparison shows how to compare STHTD-MP with GTD2-MP through the spectral radius of the deterministic Mirror-Prox error matrix. The numerical mean-operator analysis supports a faster deterministic contraction factor for STHTD-MP on the two-state, Random Walk, and Boyan Chain benchmarks, while Baird's counterexample is correctly identified as a singular boundary case where the strict comparison assumptions fail.

The stochastic experiments complement this analysis with finite-sample learning curves, standard deviations over 100 independent runs, and fair hyperparameter tuning. They show that STHTD-MP is competitive with strong online TD baselines, while also revealing that no single metric is uniformly best in every off-policy setting. Future work should develop adaptive behavior-induced metrics with theoretical guarantees and extend the analysis from fixed-policy prediction to control settings.

\section*{Acknowledgments}
This work was supported by the National Natural Science Foundation of China (Nos. 62276142, 62206133, 62202240, 62506172), the Fundamental Research Funds for the Central Universities (No. 2242025K30024), the Natural Science Foundation of Jiangsu Province (No. BK20250658), and the Scientific Research Start-up Foundation for Introduced Talents of Nanjing University of Posts and Telecommunications (No. NY225026).

\section*{Data and Code Availability}
The experimental code and configuration files are available in the public GitHub repository \texttt{GameAI-NJUPT/STHTD-MP}. The generated result tables are reproducible from the accompanying prediction-experiment script.

\bibliographystyle{elsarticle-num}
\bibliography{references}

\end{document}